\documentclass{article}

\usepackage{microtype}
\usepackage{graphicx}
\usepackage{booktabs}
\usepackage{hyperref}

\newcommand\rurl[1]{%
  \href{https://#1}{\nolinkurl{#1}}%
}

\usepackage{textcomp,gensymb}
\usepackage{tikz}
\usetikzlibrary{arrows.meta}
\usepackage{subfig}
\usepackage{float}
\usepackage{amsmath,amsfonts,amssymb,amsthm, bm}
\usepackage{mathtools}
\newtheorem{definition}{Definition}
\newtheorem{theorem}{Theorem}[section]

\usepackage[accepted]{icml2021}

\icmltitlerunning{Learning Invariances with Generalised Input-Convex Neural Networks}

\begin{document}

\twocolumn[
\icmltitle{Learning Invariances with Generalised Input-Convex Neural Networks}

\icmlsetsymbol{equal}{*}

\begin{icmlauthorlist}
\icmlauthor{Vitali Nesterov}{unibas}
\icmlauthor{Fabricio Arend Torres}{unibas}
\icmlauthor{Monika Nagy-Huber}{unibas}
\icmlauthor{Maxim Samarin}{unibas}
\icmlauthor{Volker Roth}{unibas}
\end{icmlauthorlist}

\icmlaffiliation{unibas}{Department of Mathematics and Computer Science, University of Basel, Spiegelgasse 1, 4051 Basel, Switzerland}
\icmlcorrespondingauthor{Vitali Nesterov}{vitali.nesterov@unibas.ch}
\icmlkeywords{Generalised Input-Convexity, Invariance, Cycle Consistency, Variational Autoencoder, Chemical Application}

\vskip 0.3in
]

\printAffiliationsAndNotice{}

\begin{abstract}

Considering smooth mappings from input vectors to continuous targets, our goal is to characterise subspaces of the input domain, which are invariant under such mappings. Thus, we want to characterise manifolds implicitly defined by level sets. Specifically, this characterisation should be of a global parametric form, which is especially useful for different informed data exploration tasks, such as building grid-based approximations, sampling points along the level curves, or finding trajectories on the manifold. However, global parameterisations can only exist if the level sets are connected. For this purpose, we introduce a novel and flexible class of neural networks that generalise input-convex networks. These networks represent functions that are guaranteed to have connected level sets forming smooth manifolds on the input space. We further show that global parameterisations of these level sets can be always found efficiently. Lastly, we demonstrate that our novel technique for characterising invariances is a powerful generative data exploration tool in real-world applications, such as computational chemistry.

\end{abstract}

\section{Introduction}
\label{submission}

A major goal of generative methods is the exploration of the data distribution by discovering novel and relevant examples beyond training data points. Typically, this involves some kind of inductive bias or \emph{side-information} to allow an informed exploration of a domain of our interest. For instance, in chemical applications such additional information can be a chemical property, which facilitates guided discovery of novel molecular topologies or even restricts the exploration domain to some constant property value. These subspaces of the input domain are invariant with respect to smooth mappings $f:x \mapsto y$ from input vectors $x \in \mathbb{R}^n$ to continuous targets $y$ and can be \emph{implicitly} defined by $\alpha$-level sets $L_{\alpha} = \{x : f(x) = \alpha\}$. Our goal is to find a \emph{global} parametric form for such manifolds.

\begin{figure}
\includegraphics[width=8.2cm]{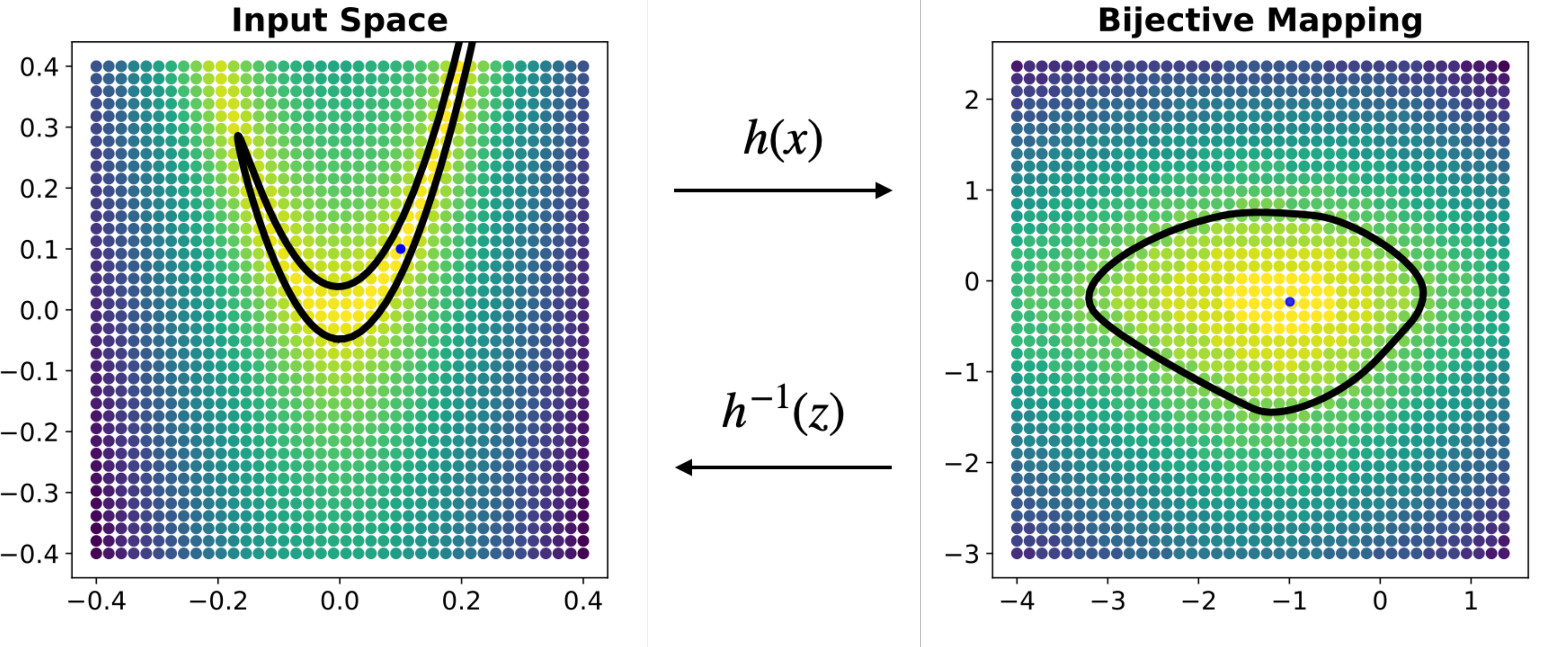}
\centering
\caption{An illustration of a global level set parameterisation on a Rosenbrock function. The left plot shows a level curve in the input space and the right plot the corresponding level curve after the bijective mapping. The global stationary points in each plot are depicted with a blue dot. Note that a global parametric form in the input space is not obvious, instead it can be easily found on its bijective mapping. For instance, by rotating a pointer placed on the stationary point. The level set points can then be transformed back into the input space.}
\label{fig:graphical-abstract}
\end{figure}

A global parameterisation of level sets is a difficult topological problem. It exists only if the level sets of $f$ are connected. Figure \ref{fig:graphical-abstract} illustrates the parameterisation of a specific level set of a $(n-1)$-dimensional manifold in the input space. The property values are defined by a Rosenbrock function forming a curved valley. Even though the level sets are connected and closed it is not obvious how to find a global parameterisation directly in the input space. However, it can be guaranteed under certain constraints that such a global parametric form exists, and that it can be found efficiently in the bijectively transformed domain. Generated points of the level set can then be transformed back to the input space by the inverse mapping $x=h^{-1}(z)$.

In this work, we introduce a novel and flexible class of neural networks that generalise input-convex networks. By construction, these networks represent functions that are guaranteed to have connected level sets. For a global parameterisation of the level sets, we choose hyper-spherical coordinates, and we proof that this type of parameterisation always exists for level sets of functions represented by our novel class of neural networks. One of the main advantages of such a global parameterisation is that interpolation paths between points on a level set can be easily constructed. Our novel approach to characterise subspaces of the input domain forms a powerful technique for informed data exploration in diverse applications.

We summarise our contributions as follows:

\begin{enumerate}
\item We propose a novel class of neural networks that implements functions which always have connected level sets that can be parameterised globally.
\item The global parameterisation themselves can be efficiently constructed using hyper-spherical coordinates. They are highly useful tools for building grid-based approximations of level sets, 
for sampling points and finding interpolation paths on the level sets.
\item Based on both synthetic and real-world datasets, we show that our approach is a highly useful tool for exploratory data analysis under side-information.
\end{enumerate}

\section{Related Work}
\label{sec:related_work}
Existing methods for learning invariances usually belong to one of two categories: (i) generative latent variable models using a partitioned latent space, and (ii) numerical level set methods that directly operate in the input space.
\subsection{Invariant Generative Models}
The multi-level variational autoencoder \cite{bouchacourt2018multi} separates the latent space into a local and global feature space, where the former is only relevant for a subgroup and the letter shares common factors of variation.

A different approach is to use adversarial networks \cite{goodfellow2014generative}, such as in \citet{creswell2017adversarial,lample2017fader}, where the discriminator learns to predict attributes and the encoder learns to prevent it. As a result, there is no information about the property in the latent space and the model relies on additional provided input to construct data with a given property. 
In \citet{klys2018learning} a mutual information regulariser is learned to isolate sub-spaces for binary targets; \citet{wieser2020inverse} extends this approach to continuous targets. The cycle consistency \cite{zhu2017unpaired} technique, as an alternative, is also used for learning disentangled representations \cite{jha2018disentangling,samarin2021learning}. 

While these models are to some extent capable of characterising invariances, they typically do not provide us a global parameterisation of the level sets, and not even the existence of a connected level set can be guaranteed in general.

\subsection{Numerical Level Set Methods}
Classical numerical level set methods \cite{osher2003level, sethian1999level} attempt to directly find level sets in the input domain $D\subset\mathbb{R}^d$ via boundary evolution.
Let $\Omega\subset D$ be the 
enclosed part of an object defined by the boundary $\partial \Omega$. The zero isovalue of a signed distance function $\phi:\mathbb{R}^d\to\mathbb{R}$, defined as
\begin{align*}
	\phi(\boldsymbol{x}) <0,  \quad& \forall \boldsymbol{x}\in\Omega\setminus \partial\Omega\\
	\phi(\boldsymbol{x})=0,  \quad & \forall  \boldsymbol{x}\in\Gamma:=\partial\Omega\\
	\phi(\boldsymbol{x}) >0,  \quad& \forall \boldsymbol{x}\in D\setminus{\overline{\Omega}}
\end{align*}
with $\overline{\Omega}:=\Omega\cup\partial\Omega$, coincides with the boundary $\Gamma$, with which we formally define a level set.
By introducing the parameter $t\in\mathbb{R}_+$ as pseudo-time, distinct level sets can be described via $\phi(\boldsymbol{x},t)=0$ of an evolving boundary $\Gamma(t)$ moving in normal direction.
The dynamics of the boundary are described by the Hamilton-Jacobi equation $\frac{\partial\phi}{\partial t}  + V\,\|\nabla \phi \| = 0$. The Hamilton-Jacobi equation is a special case of the convection equation $\frac{\partial\phi}{\partial t} + \boldsymbol{v}\cdot \nabla\phi=0$, with velocity $\boldsymbol{v}=V \frac{\nabla\phi}{\|\nabla\phi\|}$ in normal direction of the level set and a scaling factor $V$ of the velocity.
Numerical solving methods \citet{osher2003level, sethian1999level} evolve an initial boundary at $t=0$ over time for obtaining all level sets.
Although these methods usually work well in low dimensions, they are difficult to apply in higher dimensions due to their reliance on discretisations of the input domain. 
Furthermore, a numerical boundary evolution provides only approximated solution curves, but not a parametric form of the level sets.

\subsection{Preliminaries}
The method proposed in this work for providing a global parameterisation of level sets relies on the composition of two types of functions, namely bijective mapping functions and (strictly) input-convex functions, in order to obtain a certain type of generalised input convexity that implies connected level sets.

\subsubsection{Bijective Neural Networks}
\label{sec:related_work:bijections}
Normalising flows are generative models that rely on diffeomorphisms for warping base distributions into more complex ones. One of the most common applications is density estimation \cite{kobyzev2020normalizing, papamakarios2019normalizing}. 
As a result of this research direction, a wide range of flexible, efficient and trainable bijective network architectures has become readily available.
More specifically, we rely on the autoregressive networks presented in \cite{germain2015made} for our synthetic experiments.

\subsubsection{Input Convex Neural Networks}

Based on the Wasserstein Generative Adversarial Network (WGAN) \cite{arjovsky2017wasserstein} which proposes usage of Lipschitz constraints, numerous methods such as WGAN-GP \cite{gulrajani2017improved} and WGAN-LP \cite{petzka2017regularization} extend this idea to different kinds of K-Lipschitz constraints on the gradients. Spectral normalisation \cite{miyato2018spectral} globally constrains the upper bound of the Lipschitz constant on the neural network weights. \citet{sapkota2021input} proposes an invex function approximator where the gradients are regularised w.r.t. the gradient direction given a reference invex function. These techniques provide only approximate solutions, while a more interesting approach has been introduced by \citet{amos2017input}, where the network by construction guarantees input convexity providing a basis for our work. The Input Convex Neural Network (ICNN) \cite{amos2017input} constructs convex functions applying non-negativity constraints on the neural network weights and using non-decreasing activation functions. We make use of ICNNs and extend it to a more generalised notion of input convexity in neural networks.

\section{Methodology}
An essential prerequisite for the existence of a global parameterisation of level sets is the connectedness of all level sets. In the following we show that such connectivity follows from a certain convexity-type constraint imposed on the functions that can be represented by a neural network.
We then extend the discussion in the context of composite functions and the resulting properties of such, and show how to exploit these to characterise invariances in a global parametric form.

\subsection{Generalisations of Convexity}

\begin{definition}
Let $X \subseteq \mathbb{R}^n$ be an open set. The differentiable function $f: X \rightarrow \mathbb{R}$ is invex if there exists a vector function $\eta : X \times X \rightarrow \mathbb{R}^n$ such that $f(x)-f(y) \geq \eta(x,y)^T\nabla f(y),$ $\forall x,y \in X$.
\label{defone}
\end{definition}

From the definition it follows that invexity is an extension of convexity: choosing $\eta(x,y) = x-y$ in Definition \ref{defone}, invexity coincides with the definition of a differentiable convex function. An alternative definition of invexity is given by \citet{ben1986invexity} in the following theorem:
\begin{theorem}
A differentiable function $f$ is invex if and only if every stationary point is a global minimum.
\label{globalmin}
\end{theorem}

For dimension $n\geq 2$, invexity is a generalisation of pseudo-convexity, such that the set of pseudo-convex functions is strictly contained in the set of invex functions. For $n=1$, pseudo-convexity is the same as invexity. Furthermore, a differentiable pseudo-convex function has the property that all stationary points are global minima, i.e. being invex, however, there also exists invex functions that are not pseudo-convex \cite{tanaka1989generalized}.

Quasi-convexity is another generalisation of convexity. The main characteristics of a quasi-convex function is that all of its sub-level sets are convex. The class of quasi-convex functions, however, has only a partial overlap with the class of invex functions. A classical example is $f(x) = x^3$, which is quasi-convex, but not invex because the saddle-point is at $x=0$. It can be also shown that an invex function which fails to be pseudo-convex, cannot be quasi-convex either \citep{giorgi1990note}.

We now characterise the invexity of function compositions:

\begin{theorem}
  Let $X \subseteq \mathbb{R}^n$ be an open set. Let
  $g: \mathbb{R} \to \mathbb{R}$ be a differentiable strictly monotone increasing function.
  Let $h$ be a differentiable injective function $h : X \to h(X)\subseteq \mathbb{R}^n$ with differentiable inverse $h^{-1}$.
Let $f: h(X) \rightarrow \mathbb{R}$ be convex and differentiable.
Then $F=g\circ f\circ h$ is invex on $X$.
\label{thm:invex1}
\end{theorem}

Similar theorems already appeared in \citet{keller2018invexity} and \citet{mishra2008invexity}. In both cases, however, with some notation inaccuracies. Therefore, we add a formal proof here:

\begin{proof}

Convexity of $f= g^{-1}\circ F\circ h^{-1}$ implies $\forall x,y \in h(X)$:
\begin{multline}
(g^{-1}\circ F\circ h^{-1})(x)-(g^{-1}\circ F\circ h^{-1})(y) \geq \\ (x-y)^T \nabla (g^{-1}\circ F\circ h^{-1})(y)
\end{multline}

Using the chain rule, the right-hand side of this inequality can be written as:
\begin{equation}
  \label{eq:chainR}
  (x-y)^T  J_{ h^{-1}} (y)  \nabla F (h^{-1} (y)) \cdot (g^{-1})^{\prime}(F(h^{-1} (y))),
\end{equation}

where $J_{ h^{-1}}$ denotes the Jacobian matrix of the inverse transformation $h^{-1}$, and $(g^{-1})^{\prime}$ is the derivative of the scalar function $g^{-1}$.

Assume $x^\star$ is a stationary point of $F$, i.e.~$\nabla F(x^\star)=0$. Since $h$ is injective, it induces a bijection onto its image $h(X)$ with inverse $h^{-1}$. Surjectivity of $h^{-1}$ then implies that there exists $y\in h(X)$ such that $h^{-1}(y)=x^\star$. It follows that the r.h.s.~is zero and
\begin{equation}
  \label{eq:proof1}
  (g^{-1}\circ F\circ h^{-1})(x) \geq (g^{-1}\circ F)(x^\star)\; \forall x \in h(X).
\end{equation}

Since $g^{-1}$ is strictly monotone increasing, $(F\circ h^{-1})(x) \geq  F(x^\star)\; \forall x \in h(X)$.

Again, surjectivity of $h^{-1}$ implies that $\forall z\in X,\exists x\in h(X){\text{ such that }}z=h^{-1}(x)$, and therefore, 
$F(z) \geq F(x^\star)\; \forall z \in X$. Hence, every stationary point of $F$ yields a global minimum on $X$, such that $F$ is invex on $X$.
\end{proof}

It should be noted that in the setting of Theorem \ref{thm:invex1} and using the terminology in \citet{horst1984convexification}, $F=g\circ f\circ h$ can be called a ``$(h,g^{-1})$-convex'' function. Such functions belong to the family of arcwise-convex functions, all of which are known to be invex given that all domain transformations are differentiable \citep{rueda1989generalized}. 

The composition $g\circ f$ is by construction strictly quasi-convex, which follows from the strict monotonicity of $g$ and the convexity of $f$. It is also invex, which follows from the special case $h(x) = x$ in the above theorem. These properties are important to characterise invariances in a global parametric form as we intended to do.

\subsection{Parameterisation of Level Sets}
\label{param-level}

Given a differentiable function $f : \mathbb{R}^{n-1} \times \mathbb{R} \to \mathbb{R}$ and $(\theta, r) \mapsto f(\theta, r)$, let an $\alpha$-level set, induced by the equation $f (\theta, r) = \alpha$, be parameterisable if there exists a differentiable function $\gamma : \mathbb{R}^{n-1} \to  \mathbb{R}$, such that $\gamma(\theta) = r$ and $f (\theta, \gamma(\theta)) = \alpha\; \forall \theta \in \mathbb{R}^{n-1}$. In other words, we require that the solutions of $f = \alpha$ can be described as a graph of the function $\gamma$. The parameterisation itself is provided by the mapping $\theta  \mapsto (\theta, \gamma(\theta)) = \Gamma(\theta)$. Questions relating to \emph{local} parameterisation in the neighbourhood of a given point can be answered by the implicit function theorem, showing that every smooth level set can be locally parameterised if the Jacobian is invertible.

Here, we are mainly interested in \emph{global} parameterisations of level sets as it allows to interpolate between two data points on a level set and to find trajectories on the manifold induced in the input space. However, finding a global parameterisation is a difficult topological problem in general. Furthermore, a necessary condition for the existence of global parameterisations for the level set defined by $f(x, r) = \alpha$ is connectedness. In order to make $\gamma(\theta)$ well-defined everywhere, we also might want to look at compact level sets.

In general, connectedness and compactness might be difficult to characterise, however, we can benefit from invexity of our model $g\circ f \circ h$ and compute the parameterisation not directly in the input space $X$, but on its image $h(X)$ under the injective map $h$. We then can use $h^{-1}$ to map the level set back to the input space (see Figure \ref{fig:rosen-levelset-param}). The advantage is that by construction, $g\circ f$ is strictly quasi-convex. The sub-level sets of strictly quasi-convex functions are strictly convex sets and therefore the level sets are always connected.

\begin{theorem}
Suppose $f$ is differentiable, strictly convex and $\text{dom}(f) = \mathbb{R}^m$. Let  $g: \mathbb{R} \to \mathbb{R}$ be a differentiable strictly monotone increasing function. Given that $\inf ((g\circ f)(x))$ is attainable implies that all sub-level sets $L_{\alpha}(g\circ f) = \{x: (g\circ f)(x) \leq \alpha\}$ are compact $\forall \alpha \in \mathbb{R}$.
\end{theorem}

\begin{proof}
For a differentiable, strictly convex $f$, the proof is given in \citet{telgarsky2012primal}. The statement follows from observing that given strict monotonicity of $g$, $\forall \alpha \in \mathbb{R}, \exists  \alpha^\prime \in \mathbb{R}$, such that $L_{\alpha}(g\circ f) = L_{\alpha^\prime}(f)$. 
\end{proof}
Note that the compactness of the sub-level sets of $g\circ f$ implies compactness of their boundaries (i.e.~the level sets).

\begin{figure}
\centering
     \subfloat[\label{rosen-levelset-param-subfig-1:latent}]{%
       \includegraphics[width=4.0cm]{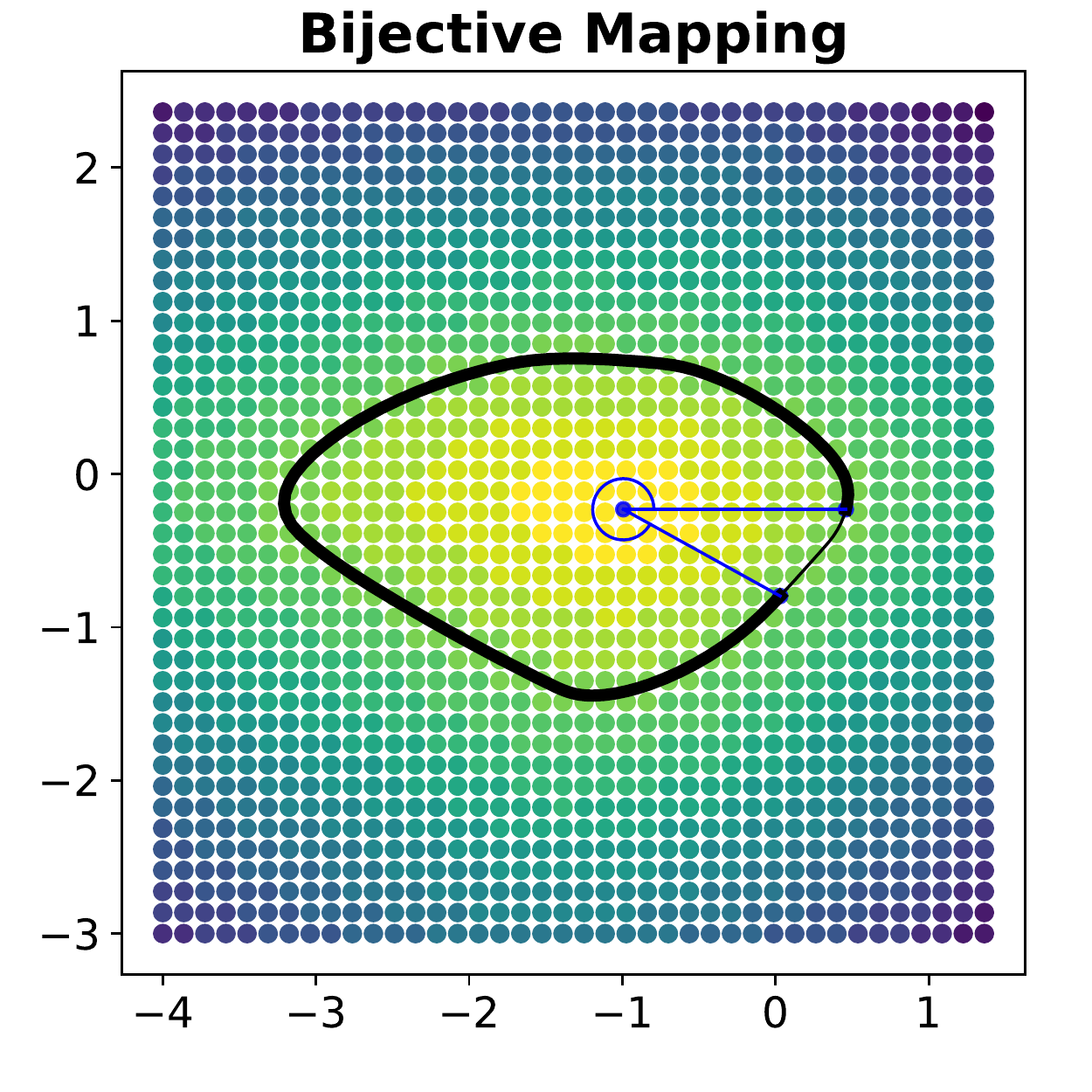}
     }
     \subfloat[\label{rosen-levelset-param-subfig-2:input}]{%
       \includegraphics[width=4.0cm]{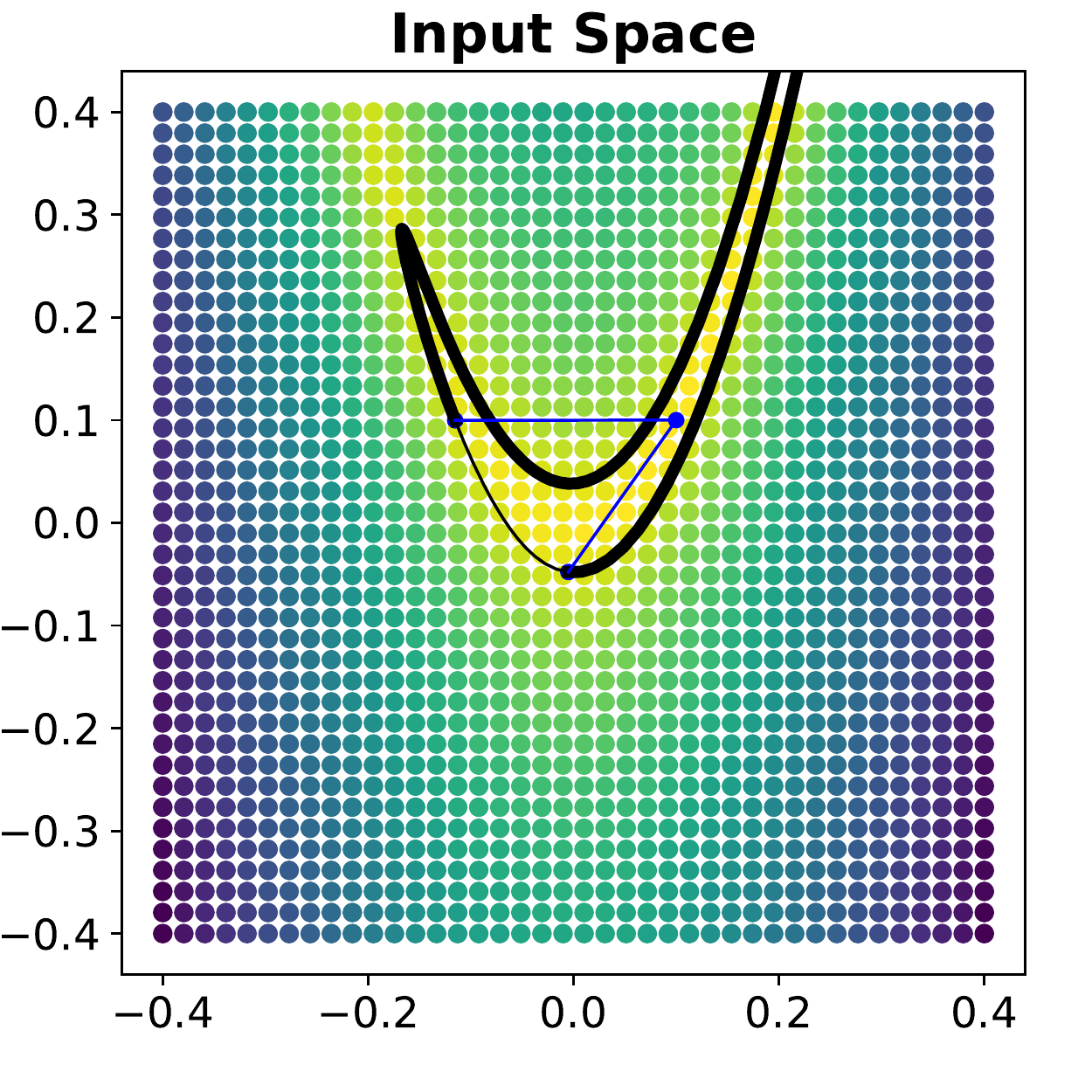}
     }
     \caption{A depiction of a level set parameterisation of the Rosenbrock function. (a) Parameterising a level curve of $g\circ f$ by first finding the minimum via gradient descent and then using polar coordinates centered at this minimum. The function values of $g\circ f$ are represented as coloured points on a grid in a two-dimensional latent space. The black curve indicates one specific level curve, parameterised in polar coordinates and found by linearly interpolating between $(\theta_i, \gamma(\theta_i))$ points on the level curve. (b) Transforming level set points back to the input space with $h^{-1}$. Note that a global parameterisation of the ``banana''-shaped curve would not be possible with polar coordinates in the input space.}
     \label{fig:rosen-levelset-param}
\end{figure}

\subsection{Retrieval of Level Curves}
\label{retrieval-level}

Given that the infimum is attainable, and therefore becomes a minimum, the strategy is as follows: We first locate the minimum of the strictly quasi-convex function $g\circ f$ in the latent space via gradient descent. According to the Theorem \ref{thm:invex1}, there will by exactly one stationary point $x^*$, which is the minimum. It defines the origin of our new coordinate system. To parameterise the $\alpha$-level set we use (hyper-)spherical coordinates in the form $\gamma_{\alpha}(\theta;  \Phi)=r$, where $r$ is the radius and $\theta \in [0,2\pi)$ the azimuthal angle, assuming that all $n-2$ latitudes $\Phi = \{\phi_i\},\, \phi \in [0,\pi]$ of the $n$-sphere are fixed. Thus, the parameterisation is provided by $(\theta,\Phi) \to ((\theta,\Phi), \gamma(\theta,\Phi)) = R(\theta,\Phi)$. To get the first point on the level curve we initialise the azimuth $\theta$ and latitudes $\Phi$ and adjust the radius $r$ via line-search given a fixed $\alpha$-level. In order to interpolate on this level set between the initial point and any other point, we linearly interpolate the angles between these two coordinates and re-adjust the radius via line-search for each intermediate interpolation step. This is very efficient, because only a sequence of one-dimensional line-searches is required to retrieve points on a level curve.

Note that this (hyper-)spherical parameterisation always works for strictly quasi-convex functions. This is because their restrictions to lines are also strictly quasi-convex, and along every line starting in the origin, the function values are strictly monotone increasing with increasing radius, which implies that there is exactly one intersection of this line and the level set. As Figure \ref{rosen-levelset-param-subfig-1:latent} shows, the pointer slides along the curve intersecting the $\alpha$-level set only once for each angle $\theta_i$, while in the input space, depicted in Figure \ref{rosen-levelset-param-subfig-2:input}, the pointer crosses the corresponding level set multiple times.

\subsubsection{Level Sets within a subspace}
In case the infimum is not attainable or we are only interested in a specific region of the input domain, the level sets within any defined subspace can be always parameterised in the same manner as described in Sec. \ref{param-level}. As shown previously, the minimum within any compact subspace $\Omega  \subseteq \mathbb{R}^n$ can be always attained under the convexity argument. Finding the minimum on a subspace is a constrained optimisation problem of a quasi-convex and smooth function, which can be solved with projected gradient descent or other techniques. Given that the minimum in $\Omega$ is attainable, the remaining steps are analogous to the strategy described before.

\subsection{Invexity of Composite Neural Networks}
\label{sec:invex_neural_network}

Given an input-invex composition of functions $g\circ f \circ h$, one can implement a bijective mapping $h: X \to h(X)$ by any type of available bijective layers, such as described in Sec. \ref{sec:related_work:bijections}. In our model we use an autoregressive flow \cite{germain2015made}. Function $g$ can be any strictly monotone increasing function. We choose $g$ to be weighted sums of simple strictly monotone increasing basis functions, such as identities or hyperbolic tangent functions, with positive weights. We consider a strictly input-convex neural network for $f$. For this, we extend the theorem on input convexity of \citet{amos2017input} to strict input convexity:
\begin{theorem}
	Define input $z=h(x)$, parameters $\tau=\{W_i^{(c)},W_i^{(z)},b_i\}$, ${i\in\left\lbrace 1,...,k-1 \right\rbrace}$, and output $c_{k}=y$. Then the function $f(z;\tau) = y$ recursively defined by
	\begin{equation}
		c_{i+1} = \sigma \left( W_i^{(c)} c_i + W_i^{(z)} z + b_i\right)
	\end{equation}
	is strictly convex in input $z$ if all weights $W_i^{(c)}$ are non-negative (and no rows of $W_{i}^{(c,z)}$ are all zeroes) and the activation function $\sigma$ is a strictly convex and increasing function.
    \label{thm:strict-convexity}
\end{theorem}
\begin{proof}
Analogous to the proof in \citet{amos2017input}, the statement follows from observing that non-negative sums of strictly convex functions remain strictly convex and that compositions of a strictly convex and increasing function with a strictly convex function preserve strict convexity. See Appendix Sec. \ref{sec:appendix_proof_strict_convexity} for more details on the proof.
\end{proof}

Our experiments with classical bijective layers show that an end-to-end training in combination with an input-convex network often leads to practical problems related to local optima and overly complex transformations that have weak extrapolation capabilities. We therefore integrate our approach into a VAE framework, where the stochasticity in the network has a regularisation effect on the bijective mapping. Here, the invertible transformation $h$ has both, the role of the encoder $\mu_z(x) = h(x)$ and the decoder $\mu_x(z) = h^{-1}(z)$. The input-convex network $f$ forms an additional decoder $\mu_y(x) = f(z)$ for the regression target $y$. That gives us a big advantage in the ability to decode the sampled latent $z$-values with the inverse mapping.

\section{Model}
\label{model}

The variational autoencoder (VAE) \cite{kingma2013auto,rezende2014stochastic} is a deep latent variable model that combines amortised variational inference (encoder) with a generative process (decoder) to learn a probabilistic generative model. The decoder network parameterises the likelihood $p(x|z)$ of the data $x$ given a latent code $z$. The encoder network maps the data $x$ to the parameters of the approximate posterior $q(z|x)$ of the latent variable $z$. The variational distribution $q$ is usually (but not necessarily) Gaussian, such that the encoder parameterises the first two moments of $q(z|x)$. To enable an informed exploration of the latent space \cite{wieser2020inverse,keller2021learning,samarin2021learning}, the model can be extended with an additional decoder for a variable $y$ as a \emph{side-information}. The likelihood $p(y|z)$ and the variational distribution $q(z|x)$ aims to approximate $p(z|x,y)$, while the encoder has only access to the input $x$.

More specifically, our model is constructed in the following manner. We parameterise the encoder network with $\phi$ and restrict it to an isotropic Gaussian $q_{\phi}(z|x) = \mathcal{N}\left(\mu_z(x), \sigma^2_z I_{\dim(z)}\right)$, where $\sigma_z$ is a trainable scale parameter and $I$ is the identity matrix. The prior $p(z) = \mathcal{N}(0, I_{\dim(z)})$ is Gaussian distributed. The joint distribution $p_{\theta,\tau}(x,y|z)$ factorises into the likelihood of the input $p_{\theta}(x|z)$ and the side-information $p_{\tau}(y|z)$ given the latent representation $z$. For a continuous variable $x$, the likelihood is an isotropic Gaussian $p_{\theta}(x|z) = \mathcal{N}\left(\mu_z(x), \sigma^2_x I_{\dim(x)}\right)$ and is parameterised with $\theta$. Since we consider a continuous target $y$, the likelihood $p_{\tau}(y|z)$ is a Gaussian and is parameterised with $\tau$. In order to train our probabilistic model, we minimise the following objective function
\begin{align}
  \mathcal{L}_\text{VAE} = -\mathbb{E}_{p(x,y)}&\mathbb{E}_{z \sim q_{\phi}(z | x)}\big[\log p_{\theta, \tau}(x,y|z)\big] \nonumber \\ 
   &+ \beta \cdot \mathbb{E}_{p(x)} D_{\mathrm{KL}}\big[q_{\phi}(z | x) \| p(z)\big]  ,
\label{eq:vae_model}
\end{align}
where $D_{\mathrm{KL}}$ denotes the Kullback-Leibler divergence and $\beta$ is a trade-off parameter to control the latent channel capacity \cite{higgins2016beta}. The gradients of the expectation are estimated using the reparameterisation trick \cite{kingma2013auto, rezende2014stochastic}. Note that for $\beta=1$, Eq.~(\ref{eq:vae_model}) coincides with the negative variational lower bound of the joint marginal log likelihood $p(x,y)$.

\subsection{Bijection through Cycle Consistency}
\label{sec:model-cycle}

For a real-world application it might be interesting to combine approximated bijective mappings with dimensionality reduction. Therefore, we alternatively use a cycle-consistency regularisation technique first introduced in the CycleGAN \cite{zhu2017unpaired}, to approximate bijection in the encoder-decoder ensemble. In this model variation, both of them form two separate networks implementing standard forward layers, where the decoder approximates the inverse of the encoder. We compute cycle-consistency loss terms on the regression variable $y$ instead of on the latent variables $z$. This is adapted from \citet{samarin2021learning}, and is aimed to avoid adverse local minima during the training process. 

Our loss terms in Eq.~(\ref{eq:vae_cycle}) are computed as follows. Given input ${x}_{obs}$ we get decoded $(x,y)$-pairs $x_{obs} \rightarrow z \rightarrow (x,y)$ and perform a cycle operation $x \rightarrow z' \rightarrow (x',y')$ to get the cycle prediction $y'$. We compute $||y - y'||_2$ on the training points to ensure consistency of input ${x}_{obs}$ and reconstruction $x$ w.r.t. property prediction $y$. In addition, we sample new latent codes $\tilde{z}$ uniformly at random and decode $\tilde{y}$ and $\tilde{y}'$ in the same manner. The loss term $||\tilde{y} - \tilde{y}'||_2$ enforces a smooth regularisation of the latent space beyond the training points.
\begin{equation}
    \mathcal{L}_\text{cycle} = \lVert y - y' \rVert _2 + \lVert \Tilde{y} - \Tilde{y}'\rVert _2
    \label{eq:vae_cycle}
\end{equation}
The complete objective function of the model is defined as follows
\begin{eqnarray}
  \mathcal{L} = \mathcal{L}_\text{VAE} + \gamma \cdot \mathcal{L}_\text{cycle},
\label{eq:complete_model}
\end{eqnarray}
where $\gamma$ constant regulates the balance between the cycle and the reconstruction loss terms.

\section{Experiments}

Our experiments are organised in the following manner: First, we evaluate the capabilities of our model on synthetic data and show differences to the $\beta$-VAE \cite{higgins2016beta} with a cycle-consistency regularisation. Second, we evaluate our model on a modified version of Fashion-MNIST as well as on a chemical application. 

\subsection{Synthetic Dataset}

The experiments on synthetic data provide an intuition for our model and show how we parameterise the level sets. Thereby, we focus on learning input invexity on a two-dimensional input. We consider these experiments as they allow to visualise the data and help to understand the idea on the first glance. The focus is on learning and visualising input invexity on a two-dimensional example.
For our first experiment, we use a modified Rosenbrock function defined as 
\begin{equation}
  \label{eq:rosen}
  f(x_1,x_2) =  \sqrt[4]{ \left( 1-10 x_1\right)^2 + 100 \left( 10 \left( x_2 - x_1^2 \right) \right)^2 },
\end{equation}
which is a differentiable invex function with one stationary point, i.e. the global minimum, at $x_1= x_2= 0.1$. For the second set of experiments, we use a mixture of two Gaussians with means $\mu_1 = -0.2$, $\mu_2 = 0.2$ and (diagonal) covariances $\sigma_1^2=\sigma_2^2=0.02\cdot$Id.
The property $y$ is multiplied by $-1$ to invert maxima to minima. For both data sets, we sample 1,600 data points $x \in \mathbb{R}^{2}$ on the $[-0.4, 0.4]^{2}$
grid and get the corresponding one-dimensional property values $y$ with either the Rosenbrock function or the Gaussian mixture.

For the experiments on synthetic data, we use two different models. As a baseline model, we use a $\beta$-VAE with approximated bijective mappings through cycle-consistency. The generalised input-convex model is based on the $\beta$-VAE, where the encoder is by construction bijective and the property decoder is a strict input-convex neural network (see Sec. \ref{sec:invex_neural_network}). See Appendix Sec. \ref{sec:appendix_synthetic_experiment} for more details on the datasets and the model architecture.

\begin{figure}
\centering
     \subfloat[\label{subfig-1:latent}]{%
       \includegraphics[width=4.0cm]{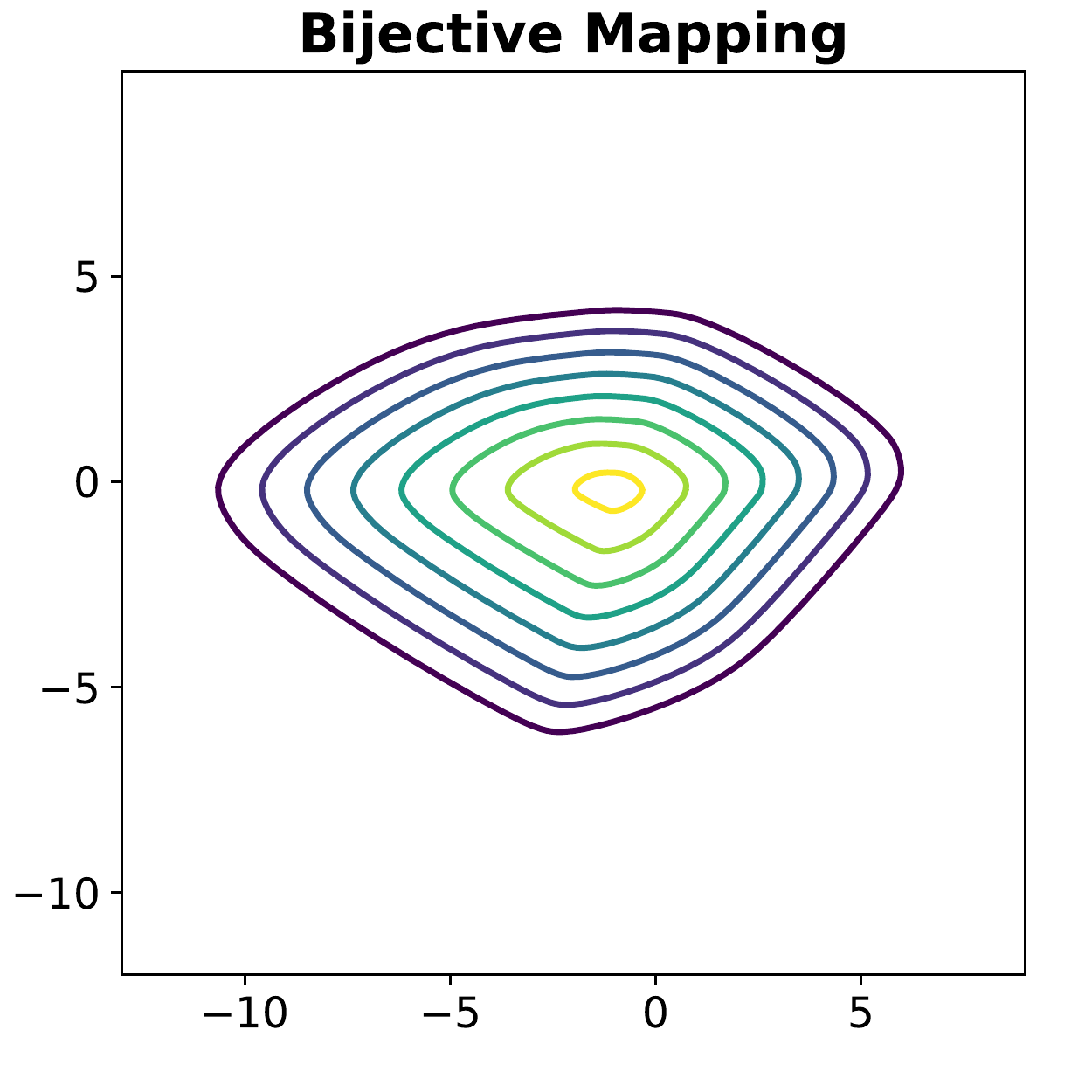}
     }
     \subfloat[\label{subfig-2:input}]{%
       \includegraphics[width=4.0cm]{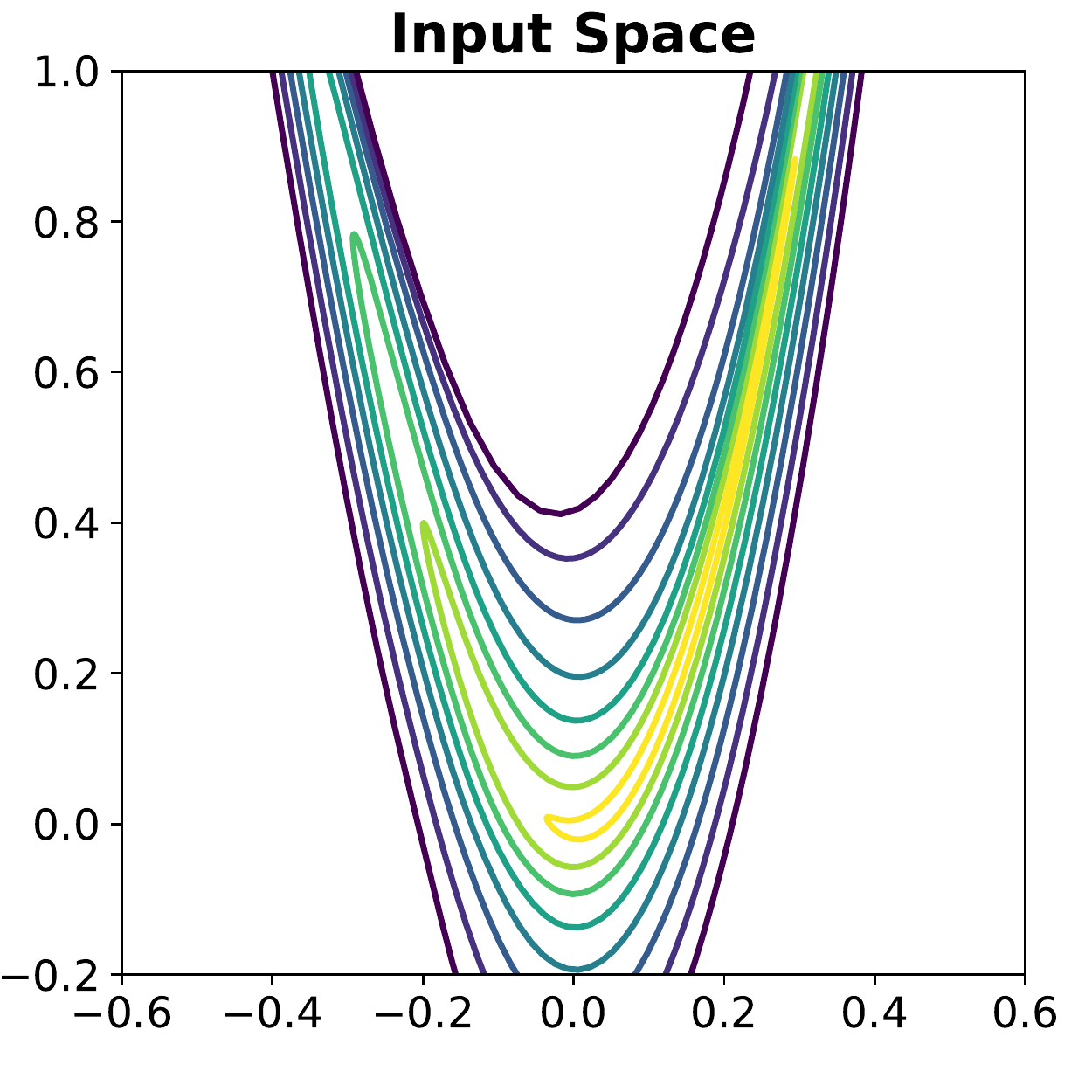}
     }
     \caption{A depiction of multiple level set parameterisations of the Rosenbrock function. (a) Parameterising multiple level curves on the bijective mapping. (b) Transforming level set points back to the input space with the inverse of the encoder. Note that the model has very strong extrapolation capabilities predicting smoothly the ``banana''-shaped valley far beyond the training data range.}
     \label{fig:rosen-levelset-extrapolate}
\end{figure}

In the first experiment, we demonstrate our results on the dataset sampled from the two-dimensional Rosenbrock function with one stationary point. To this end, we compare the baseline to the generalised input-convex model. Both of the models achieve similar results concerning the property prediction and input reconstruction. Since our model uses a bijective encoder, the encoder mapping can be directly inverted back to the input space.
Despite comparable results, the baseline model provides no guaranties for the existence of a global parametric form of the level sets. Hence, we focus only on the generalised input-convex model to demonstrate a global parameterisation of the level sets. In Figure \ref{fig:rosen-levelset-extrapolate} we show multiple results of level sets. To plot the level curves, we sample points along the level curves of the bijective mapping and transform them back to the input space with the inverse of the encoder $h^{-1}$. Note that the model extrapolates the function shape far beyond the $[-0.4, 0.4]^2$ training range, smoothly forming the typical ``banana''-shaped valley.

\begin{figure}%[h]
\centering
     \subfloat[\label{subfig-1:dummy}]{%
       \includegraphics[width=4.0cm]{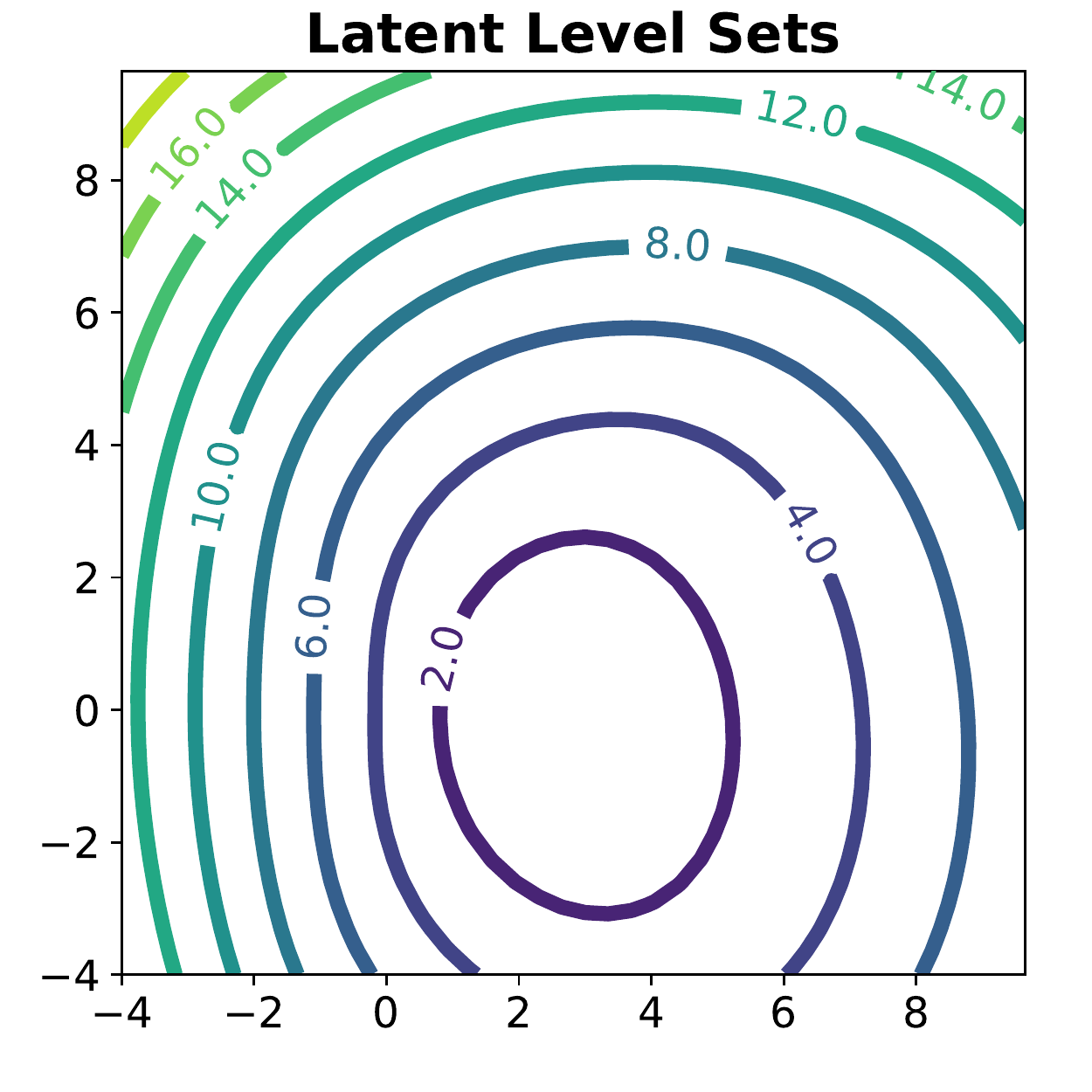}
     }
     \subfloat[\label{subfig-2:dummy}]{%
       \includegraphics[width=4.0cm]{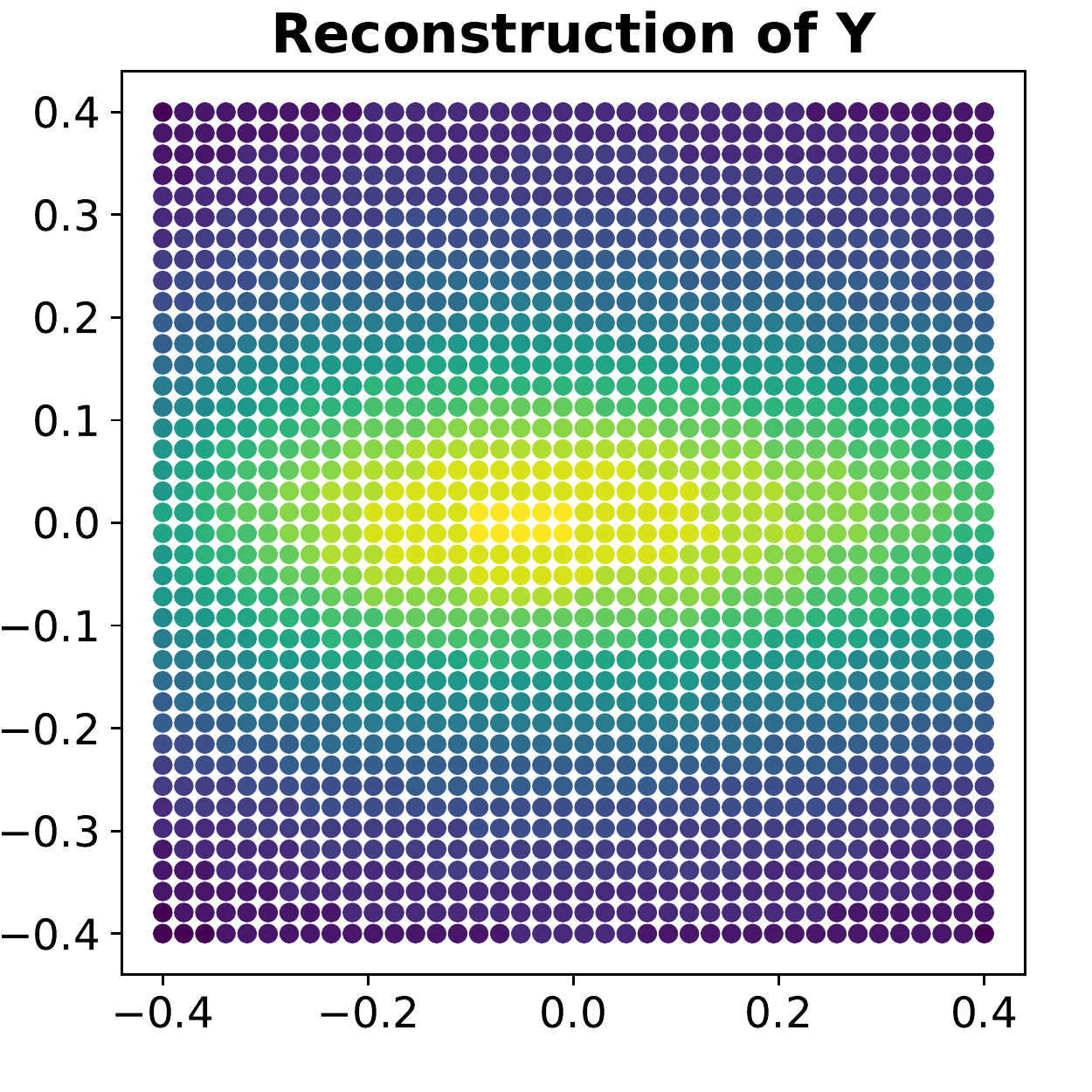}
     }
     \hfill
     \subfloat[\label{subfig-3:dummy}]{%
       \includegraphics[width=4.0cm]{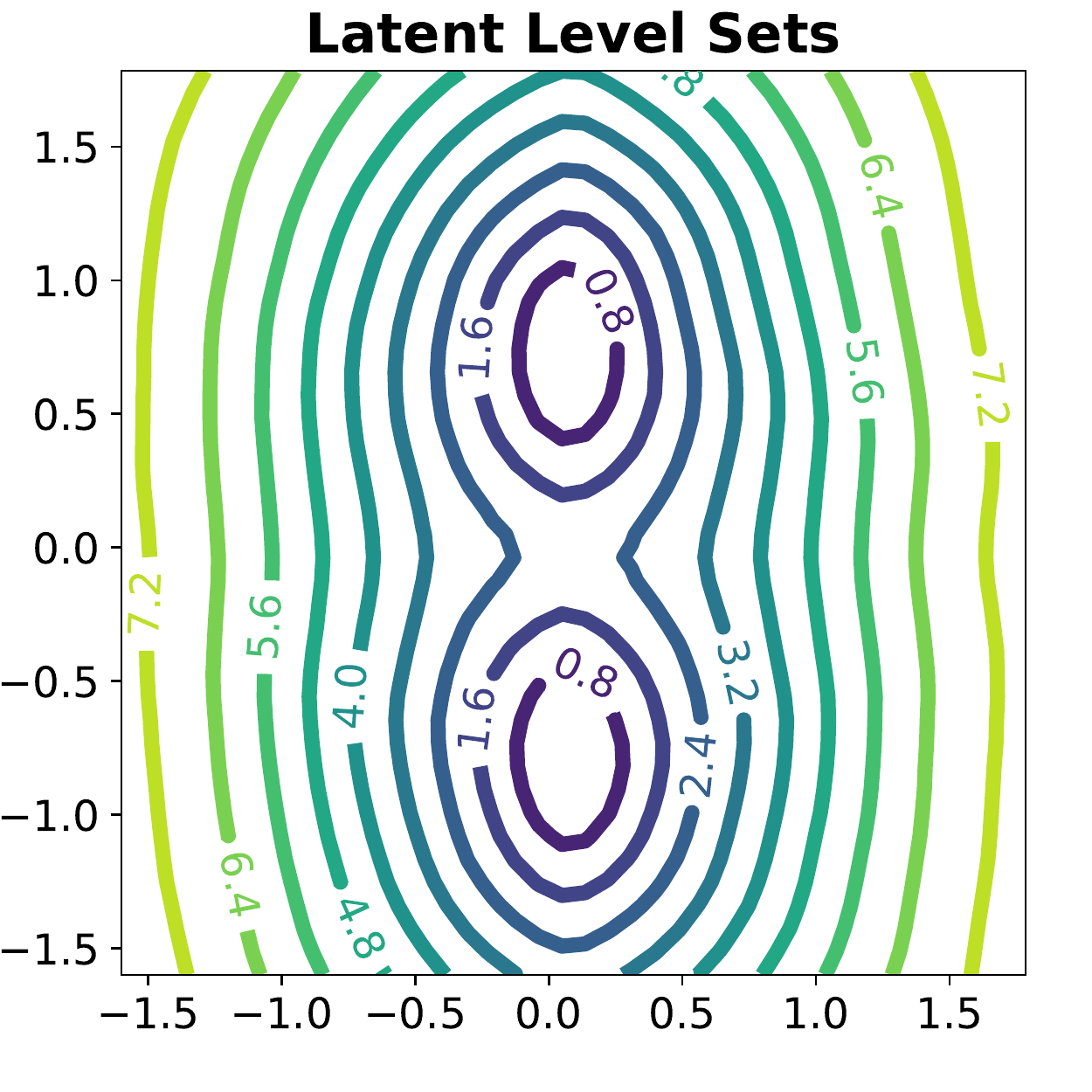}
     }
     \subfloat[\label{subfig-4:dummy}]{%
       \includegraphics[width=4.0cm]{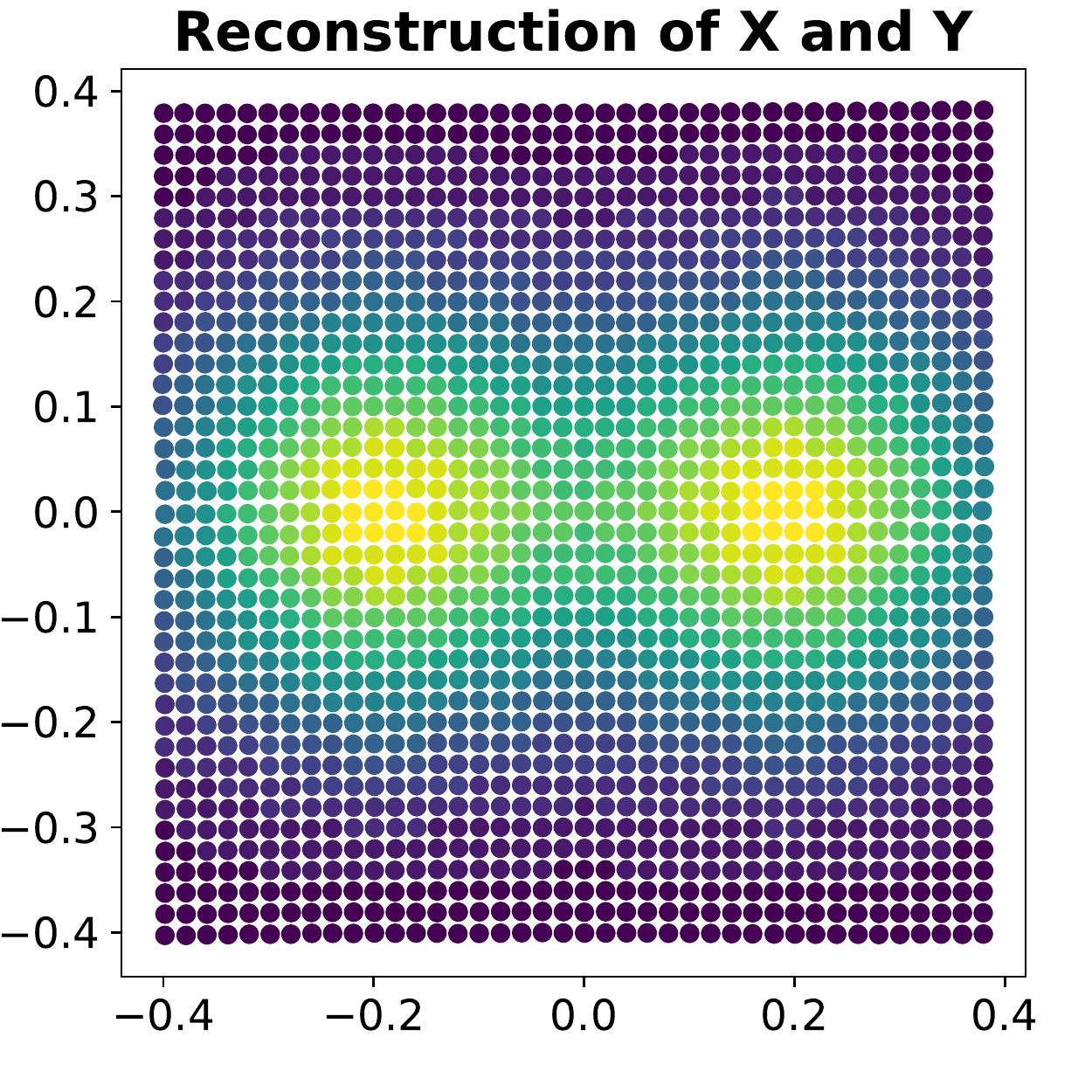}
     }
     \caption{A depiction of the difference between the latent space arrangement of the $\beta$-VAE with cycle-consistency regularisation and our model. (a) Depicts parameterised level sets on the input convex latent space and (c) shows parameterised level sets on the unconstrained latent space. (b) and (d) depicts the corresponding reconstruction of the $(x,y)$-pairs.}
     \label{fig:gauss-levelset}
\end{figure}

To further emphasise the differences between the baseline and our generalised input-convex model, we conduct a second experiment on the Gaussian mixture dataset with two stationary points. Figure \ref{fig:gauss-levelset} illustrates our results. Again, the baseline model is able to reconstruct the $(x,y)$-pairs, however, the latent representation is not convex on $x$ and it naturally contains two stationary points. Identifying one of them, only a local level set parameterisation is possible, i.e. we obtain a curve in the input space around one of the minima. By definition, our model is not able to predict property $y$ accurately, given more than one stationary point in the input space. This is due to the mapping of two minima to a single stationary point in the latent representation, 
leading to a poor property prediction loss. However, this result is of particular interest as it allows detecting disconnected level sets due to multiple minima in the input space. This information can serve as an indication that a global parameterisation of level sets simply does not exist.

\subsection{Rotated Fashion-MNIST}

As a more challenging dataset, we use a modified version of the Fashion-MNIST \cite{xiao2017fashion}. In order to create a continuous property value $y$ for each item, we rotate the images uniformly at random, $\mathcal{U}\large([-45^{\circ},45^{\circ}]\large)$, around their centre. The images are zero-padded to a size of $36\times36$ pixels to avoid cropping artefacts as a side effect of the rotation. We also excluded bags, ankle boots and sandals from the dataset. Due to a 1296-dimensional input space, we make use of the pseudo-bijective VAE (see Sec. \ref{sec:model-cycle}) to learn a more compact 40-dimensional latent representation.
Further details on the data and the model architecture are provided in Appendix Sec. \ref{sec:appendix_fmnist_experiment}.

In the experiment, we interpolate between two test images of the same rotation angle but different class labels, i.e. between a pullover and trousers. Our level set of choice is a $36^{\circ}$ counter clockwise rotation. In Figure \ref{fig:fmnist-interpolate}, we demonstrate that the objects on the level set smoothly transform from a pullover to trousers, but preserve the fixed rotation angle. 

\begin{figure}
\vspace*{0.15in}
\includegraphics[width=7.5cm]{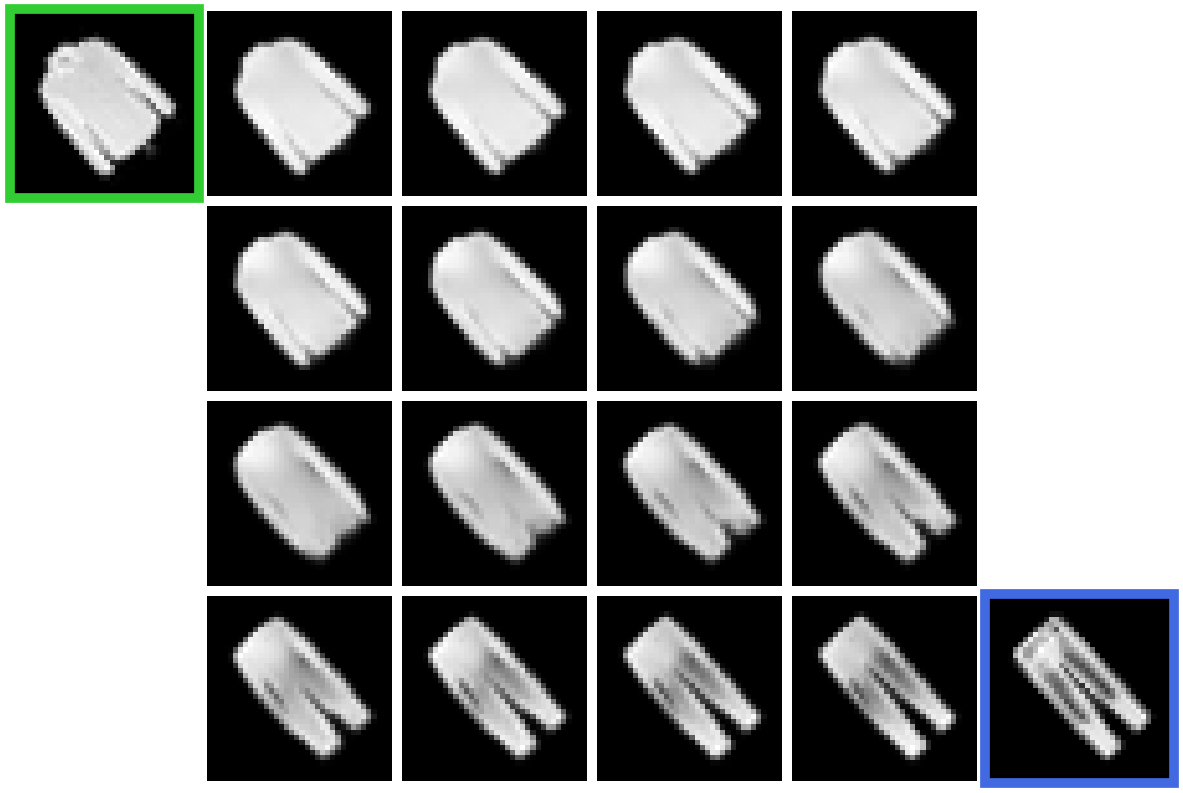}
\centering
\caption{Interpolation on the level set between two test-images from the Fashion-MNIST dataset. The left- and right-most images correspond to the reference images. In between, the reconstructions alongside the interpolation path on the level set are depicted.}
\label{fig:fmnist-interpolate}
\end{figure}

\subsection{Small Organic Molecules}

\begin{figure*}[ht]
\vspace*{0.0in}
\includegraphics[width=17.0cm]{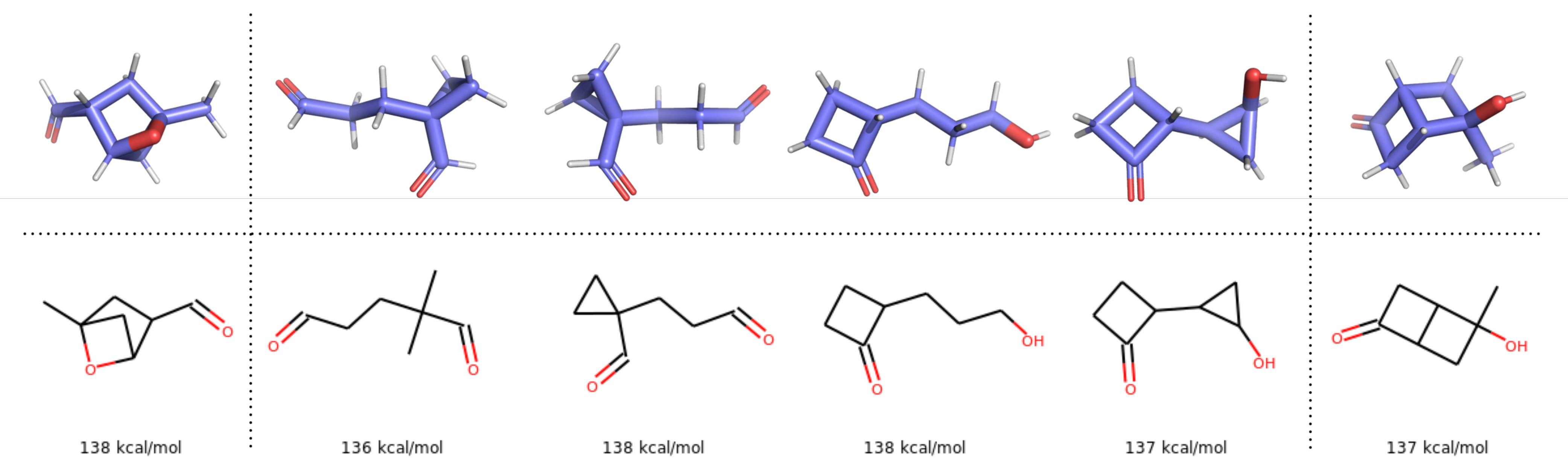}
\centering
\caption{Interpolation between two molecules on a fixed level set of 138 $\textrm{kcal} \textrm{ mol}^{-1}$ for the band gap energy. The two reference molecules are positioned in the first and last column. Additionally to each generated graph (bottom row) we plot the corresponding spatial structure (top row) for identified molecules in our dataset. The label of each molecule indicates the true property value.}
\label{fig:molecule-experiment}
\end{figure*}

As a real-world application, we use the QM9 dataset \cite{rama2014} which includes 133,885 small organic molecules. The molecules consist of up to nine heavy atoms (C, O, N and F), not including hydrogen atoms. Each molecule includes chemical properties which are calculated with the Density Functional Theory methods. In our experiments, we select a subset with a fixed stoichiometry $(C_7 O_2)$ which consists of 16,302 molecules. We choose the band gap energy as the property of our interest. 

We represent the molecules as a graph $G = (V, E)$, providing atom types and molecular bonds. The node set $V$ consists of up to 9 heavy atoms. The weighted edges set $E$ is represented as an upper triangular weighted adjacency matrix. Both graph components are one-hot encoded and include an additional blank symbol. We further employ a canonical ordering \cite{o2012towards} of the atoms for a standardised enumeration among all the molecules in the dataset \cite{nesterov20203dmolnet}. In total, the input space involves 189 dimensions. To learn a low-dimensional representation we approximate the bijective mapping (see Sec. \ref{sec:model-cycle}), such that our latent space consists of 22 dimensions. To improve the performance, we compute the cycle loss on the differentiable approximation of the one-hot categorical distribution. See Appendix Sec. \ref{sec:appendix_qm9_experiment} for more details on the model architecture.

Our training set consists of 13,800 molecules. On the test set of 2,500 molecules, we are able to reconstruct 97 \% of the graphs. For prediction of the band gap energies a MAE of 4.5 $\textrm{kcal} \textrm{ mol}^{-1}$ is achieved. We compute the invariance as the MAE between the predicted property $y$ and the cycle prediction $y'$, for which we achieve 0.4 $\textrm{kcal} \textrm{ mol}^{-1}$ on the test set. In an additional experiment, we interpolate between two reference molecules as shown in Figure \ref{fig:molecule-experiment}. To do so, we first identify the largest cluster on the band gap energy values in our dataset. We then compute the mean of these cluster, which is approximately 138 $\textrm{kcal} \textrm{ mol}^{-1}$ and define this as our level set of interest. Subsequently, we choose two reference molecules with a band gap close to 138 $\textrm{kcal} \textrm{ mol}^{-1}$ and compute the interpolation path on the manifold spanned by the angles. In Figure \ref{fig:molecule-experiment}, unique molecular graphs are depicted  that can be found on the interpolation path. All these molecules could be retrieved in the QM9 database, such that we also know their ground-truth energies. Note that the indicated true property values are close to the fixed level set and are within the range of the model error.

\section{Conclusion}

This paper introduces a novel and flexible class of neural networks that generalise input-convex networks. Building on the notion of invexity, we have shown in an extensive proof that such networks learn mappings that imply connected level sets by construction. We motivate this requirement, as it guarantees the existence of a global parametric form of level sets. Our experimental results on synthetic and benchmark data, as well as for a real-world application in computational chemistry, demonstrate that such connected level sets are indeed learned and can be efficiently parameterised. This allows a practical interpolation between data points on a level set. As an additional feature, our model is also able to detect the absence of connected level sets in the input data. We further demonstrate that our network can be easily integrated into a generative framework. This enables characterising invariances for informed data exploration, providing a highly powerful tool for exploratory data analysis. On a chemical example, we show complex topological transformations of molecular structures on the level set of a fixed chemical property. The limitation of our method to single stationary points can be addressed by learning multiple convex functions on sub-domains in future work. Further extensions may include new kinds of bijective mappings and likelihood-based training with Normalising Flows.

\newpage
\bibliography{generalised_icnn}
\bibliographystyle{icml2021}

\newpage
\,
\newpage
\appendix

\section{Appendix}
\subsection{Proof of Strict Input Convexity}
\label{sec:appendix_proof_strict_convexity}

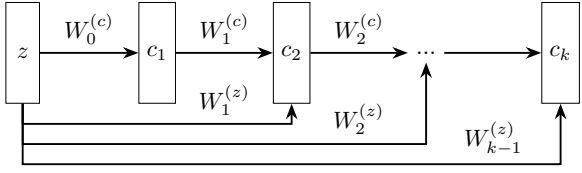
\begin{figure}[h]
	\begin{center}
		\resizebox{0.45\textwidth}{!}{%		
		\begin{tikzpicture}
			\node (z) at (0,0) [rectangle,draw,minimum height=15mm,minimum width=5mm] {$z$};
			\node (c1) at (2,0) [rectangle,draw,minimum height=15mm,minimum width=5mm] {$c_1$};
			\node (c2) at (4,0) [rectangle,draw,minimum height=15mm,minimum width=5mm] {$c_2$};
			\node (c) at (6,0) [rectangle,minimum height=0mm,minimum width=5mm] {$...$};
			\node (ck) at (8,0) [rectangle,draw,minimum height=15mm,minimum width=5mm] {$c_k$};
			\draw[-Stealth, thick] (z) -- (c1) node[midway,above]{$W_0^{(c)}$};
			\draw[-Stealth, thick] (c1) -- (c2) node[midway,above]{$W_1^{(c)}$};
			\draw[-Stealth, thick] (c2) -- (c) node[midway,above]{$W_2^{(c)}$};
			\draw[-Stealth, thick] (c) -- (ck);
			\draw[-Stealth, thick] (z.south) -- ++(0,-0.3) -- ++(4,0) -- (c2.south) node[shift={(-1,-0.3)},above]{$W_1^{(z)}$};
			\draw[-Stealth, thick] (z.south) -- ++(0,-0.6) -- ++(6,0) -- (c.south) node[shift={(-1,-1.2)},above]{$W_2^{(z)}$};
			\draw[-Stealth, thick] (z.south) -- ++(0,-0.9) -- ++(8,0) -- (ck.south) node[shift={(-1,-0.9)},above]{$W_{k-1}^{(z)}$};
		\end{tikzpicture}
		}
	\end{center}
	\caption{Illustration of the (strictly) Input Convex Neural Network. The figure is adapted from \citet{amos2017input}.}
\end{figure}
We use the following notation:
\begin{itemize}
	\item Input $z\in\mathbb{R}^d$ with $d$ input dimensions, which in the main text is given by the bijective mapping $z=h(x)$
	\item Hidden activations $c_i\in\mathbb{R}^{l_i}$, with $l_i$ being the width of the $i$-th layer and $c_k$ being the output $y$, i.e. $c_k=y$
	\item $k$ layers, ${i\in\left\lbrace 0,1,...,k-1 \right\rbrace}$
	\item Weights $W_i^{(c)}\in\mathbb{R}_{\ge 0}^{l_{i+1}\times l_{i}}$, only defined for $i>0$ or $W_0^{(c)}=0$
	\item Input injection weights $W_i^{(z)}\in\mathbb{R}^{l_{i+1} \times d}$ and bias terms $b_i \in \mathbb{R}^{l_i}$
	\item Element-wise activation function $\sigma=(\sigma_1, ..., \sigma_{l_{i+1}})$, i.e. $\sigma_j$: $\mathbb{R}\to\mathbb{R}$ 
	\item Parameters $\tau=\left\lbrace W^{(z)}_{0:k-1}, W^{(c)}_{1:k-1}, b_{0:k-1} \right\rbrace$
\end{itemize}
Let us restate the theorem in this notation:

\paragraph{Theorem \ref{thm:strict-convexity} (restated).} {\it
	The function $f(z;\tau) = y$ recursively defined by
	\begin{equation}
		c_{i+1} = \sigma \left( W_i^{(c)} c_i + W_i^{(z)} z + b_i\right)
	\end{equation}
	is strictly convex in input $z$ if all weights $W_i^{(c)}$ are non-negative (and no rows of $W_i^{(c,z)}$ are all zeroes), and the activation function $\sigma$ is a strictly convex and increasing function.
}
\paragraph{Remark.} 
We need to show four parts, where all statements are w.r.t. input $z$:
\begin{itemize}
	\item[1.] Non-negative sums of strictly convex functions remain strictly convex;
	\item[2.] Sums of strictly convex functions and affine functions preserve strict convexity;
	\item[3.] Composition of a strictly convex and increasing function with a strictly convex function remains strictly convex;
	\item[4.] Composition of a strictly convex function with an affine function is strictly convex (first layer).
\end{itemize}
\begin{proof}
    The proofs of these statements follow from extension of standard properties of convex functions to strict convexity \citep{boyd2004convex}. As the activation function operates element-wise, we can restrict ourselves to the case where the resulting function maps to the reals, $f:\mathbb{R}^d \to \mathbb{R}$. Therefore, let $a,b\in\mathbb{R}_{> 0}$, $\lambda\in[0,1]$ and $x,y\in\mathbb{R}^d$, $x\neq y$, in the following.
    
	\underline{Part 1:} Let $f=ag+bh$, with $g:\mathbb{R}^d \to \mathbb{R}$ and $h:\mathbb{R}^d \to \mathbb{R}$ being strictly convex functions, i.e. 
	\begin{equation}
		g\big(\lambda x + (1-\lambda)y\big) < \lambda g(x) + (1-\lambda) g(y),
	\end{equation}
	with the same holding true for $h$. For $f$ we obtain
	\begin{align}
		f\big(\lambda x + (1-\lambda)y\big) =&ag\big(\lambda x + (1-\lambda)y\big) \nonumber \\
		&+ bh\big(\lambda x + (1-\lambda)y\big) \\
		<&\lambda ag(x) + (1-\lambda) ag(y) \nonumber \\ 
		&+ \lambda bh(x) + (1-\lambda) bh(y) \\
		=&\lambda \big(ag(x) + bh(x)\big) \nonumber \\
		&+ (1-\lambda) \big(ag(y) + bh(y)\big) \\
		=&\lambda f(x) + (1-\lambda) f(y)
	\end{align} 
	which proves that $f$ is strictly convex. 
	
	\underline{Part 2:} Let $f=ag+h$, with $g:\mathbb{R}^d \to \mathbb{R}$ being strictly convex functions and $h:\mathbb{R}^d \to \mathbb{R}$ being an affine function, i.e.
	\begin{equation}
		h(x) = Cx + d
	\end{equation}  
	with $C\in\mathbb{R}^{1\times d}$ and $d\in\mathbb{R}$.
		For $f$ we obtain
	\begin{align}
		f\big(\lambda x + (1-\lambda)y\big) =& ag\big(\lambda x + (1-\lambda)y\big) \nonumber \\
		&+ h\big(\lambda x + (1-\lambda)y\big) \\
		<& \lambda ag(x) + (1-\lambda) ag(y) \nonumber \\
		&+  C\big(\lambda x + (1-\lambda)y\big) + d \\
		=& \lambda ag(x) + (1-\lambda) ag(y) \nonumber \\
		&+ \lambda Cx + (1-\lambda)Cy \nonumber \\ 
		&+ \lambda d + (1-\lambda) d\\
		=& \lambda \big(ag(x) + Cx + d \big) \nonumber \\
		&+ (1-\lambda) \big( ag(y) + Cy + d\big) \\
		=& \lambda f(x) + (1-\lambda) f(y)
	\end{align}
	which proves that $f$ is strictly convex. 
	
	\underline{Part 3:} Let $f = h \circ g$, with $g:\mathbb{R}^d \to \mathbb{R}$ and $h:\mathbb{R} \to \mathbb{R}$ (the activation function) being strictly convex functions. As before, we have
	\begin{equation}
		g\big(\lambda x + (1-\lambda)y\big) < \lambda g(x) + (1-\lambda) g(y).
	\end{equation}
	for both $g$ and $h$. We can show that
	\begin{align}
		f\big(\lambda x + (1-\lambda)y\big) &= h\big(g\big(\lambda x + (1-\lambda)y\big)\big) \\
		&< h\big( \lambda g(x) + (1-\lambda) g(y) \big) \\
		&< \lambda h\big(g(x)\big) + (1-\lambda) h\big(g(y)\big) \\
		&= \lambda f(x) + (1-\lambda) f(y)
	\end{align} 
	which again provides that $f$ is strictly convex.
	
	\underline{Part 4:} For the first layer,
	\begin{equation}
		z_{1} = \sigma \left(W_0^{(x)} x + b_0\right),
	\end{equation}
	we need to show that $f=h\circ g$ is strictly convex with $h:\mathbb{R} \to \mathbb{R}$ (the activation function) being strictly convex, and $g:\mathbb{R}^d \to \mathbb{R}$ being the affine function $g(x)=Cx+d$ with $C\in\mathbb{R}^{1\times d}$ and $d\in\mathbb{R}$ as before. Furthermore, we require that $g$ is not just constant, i.e. $C\neq (0,...,0)$ (no rows are all zeroes). For an affine function we can show that
	\begin{align}
		g\big(\lambda x + (1-\lambda)y\big) =& C\big(\lambda x + (1-\lambda)y\big) + d \\
		=& \lambda Cx + (1-\lambda)Cy \nonumber \\
		&+ \lambda d + (1-\lambda)d \\
		=& \lambda \big(Cx + d\big) \nonumber \\
		&+ (1-\lambda) \big( Cy + d \big) \\
		=& \lambda g(x) + (1-\lambda) g(y),
	\end{align}
	i.e. the affine function is both convex and concave. Therefore,
	\begin{align}
		f\big(\lambda x + (1-\lambda)y\big) &= h\big( g\big(\lambda x + (1-\lambda)y\big)\big) \\
		&= h \big(\lambda g(x) + (1-\lambda) g(y) \big) \\
		&<  \lambda h\big(g(x)\big) + (1-\lambda) h\big(g(y)\big) \\
		&= \lambda f(x) + (1-\lambda) f(y),
	\end{align}
	i.e. $f$ is strictly convex. 
\end{proof}

\subsection{Parameterisation of Level Sets}
Let $X \in \mathbb{R}^n$ be cartesian coordinates in a n-dimensional Euclidean space that are mapped to a (hyper-)spherical coordinate system  $\{r, \varphi_1, \varphi_2, \dots, \varphi_{n-1}\}$ with
\begin{align}
    x_i &\in \mathbb{R} \quad &&\forall i \in \{1, 2, \dots, n\}\\
    r &\in \mathbb{R}_+\\
    \varphi_i &\in [0,\pi] &&\forall i \in \{1, 2, \dots n-2\} \\
    \varphi_{n-1} &\in [0,2\pi)
\end{align}

Then the mapping to and from a spherical coordinate system for a n-dimensional space is defined as follows \cite{blumenson1960derivation}:
\begin{align}
x_{1} &=r \cos \left(\varphi_{1}\right) \\
x_{2} &=r \sin \left(\varphi_{1}\right) \cos \left(\varphi_{2}\right) \\
& \vdots \\
x_{n-1} &=r \sin \left(\varphi_{1}\right) \cdots \sin \left(\varphi_{n-2}\right) \cos \left(\varphi_{n-1}\right) \\
x_{n} &=r \sin \left(\varphi_{1}\right) \cdots \sin \left(\varphi_{n-2}\right) \sin \left(\varphi_{n-1}\right) \\
\varphi_{1} &=\arccos \left( \frac{x_{1}}{\sqrt{x_{n}^{2}+x_{n-1}^{2}+\cdots+x_{1}^{2}}} \right) \\
\varphi_{2} &=\arccos \left( \frac{x_{2}}{\sqrt{x_{n}^{2}+x_{n-1}^{2}+\cdots+x_{2}^{2}}} \right) \\
& \vdots \\
\varphi_{n-2} &=\arccos \left( \frac{x_{n-2}}{\sqrt{x_{n}^{2}+x_{n-1}^{2}+x_{n-2}^{2}}} \right) \\
\varphi_{n-1} &=2 \operatorname{arccot} \left( \frac{x_{n-1}+\sqrt{x_{n}^{2}+x_{n-1}^{2}}}{x_{n}} \right)
\end{align}

\subsection{Synthetic Experiments}
\label{sec:appendix_synthetic_experiment}

\subsubsection{Architecture and training}

The baseline model is a $\beta$-VAE with approximated bijective mappings through cycle-consistency. For the encoder and decoder, we use two fully-connected hidden layers with 1024 neurons each. The property decoder consists of two fully-connected hidden layers with 512 neurons each. Each of the hidden layers use a ReLU activation function. The number of latent dimensions is set to $d_z=2$. The model further involves cycle-consistency regularisation as described in Sec. \ref{sec:model-cycle}, to approximate bijective mappings.

The generalised input-convex model is based on the $\beta$-VAE. The encoder is by construction bijective and consists of four autoregressive flow layers with 128 neurons each and an ELU activation. Since the encoder is directly invertible, a separate decoder network is not included. The property decoder is a strictly input-convex neural network (see Sec. \ref{sec:invex_neural_network}). It consists of four hidden layers with 512 units each and softplus as a strictly convex activation.

The generalised input-convex model with an approximated bijection is based on the $\beta$-VAE. For the encoder and decoder, we use two fully-connected hidden layers with 1024 neurons each. The property decoder consists of two fully-connected hidden layers with 512 neurons each. Each of the hidden layers use a ReLU activation function. The property decoder is a strictly input-convex neural network (see Sec. \ref{sec:invex_neural_network}). It consists of four hidden layers with 512 units each and softplus as a strictly convex activation.

All of the models are trained with an Adam optimiser, a learning rate of $10^{-4}$, and a batch size of 250. The $\beta$ parameter is annealed by a factor of $0.99$ after every 30 epochs and the $\gamma$ constant is set to $0.01$.

\subsection{Rotated Fashion-MNIST: Architecture and Training}
\label{sec:appendix_fmnist_experiment}
For the rotated Fashion-MNIST experiment we use the following general architecture and training settings. For the full details of the implementation we refer to the code provided.

\textbf{Encoder.} The encoder consists of an initial $7\times7$ convolution layer, followed by six $3 \times 3$ convolution layers increasing number of filters.
For the down-sampling, three of the $2\times2$ convolutions have a stride of 2. In-between the non-striding convolutions, Batch Normalisation  \cite{ioffe2015batch} was used. The flattened output is fed into a fully connected network that parameterises mean and standard deviation of the latent distribution $q(z|x)$. 

\textbf{Image Decoder.} The image decoder consists of 3 transposed convolutions with a kernel size of $3 \times 3$ and a stride of $2$.
Each transposed convolution is followed by $3\times 3$ convolutions with Batch Normalisation in between. 
The image likelihood, i.e. the reconstruction loss, is Gaussian with mean given by the decoder network and the covariance as a $\sigma I$, with $\sigma$ being a learnable parameter, $I$ the identity matrix.

\textbf{Property Decoder.} The property decoder is a 5 layered fully input convex neural network followed by a hyperbolic tangent activation.
The output is then scaled to the range from $-\pi /4$ to $\pi/4$ (i.e. the range of possible rotations).
As activation function we use softplus.

\textbf{Training.} We train for $30$ epochs with Adamax \cite{kingma2014adam}, using a mini-batch size of $32$ and a latent space dimension of $35$. Weights of the two reconstruction losses, as well as the cycle-consistency loss were determined experimentally. 
For the weight of the KL-Divergence, we started with an initial value of $1$ and slowly decreased it using an exponential cooling function, multiplying the weight by $0.95$ each 20 mini-batches.

\subsection{Chemical Experiments}
\label{sec:appendix_qm9_experiment}

The QM9 dataset \cite{rama2014} includes 133,885 organic molecules. The molecules consist of up to nine heavy atoms (C, O, N, and F), not including hydrogen. Each molecule includes corresponding geometric, energetic, electronic and thermodynamic properties chemical properties. These properties have been computed with the Density Functional Theory methods. In our experiments, we select a subset with a fixed stoichiometry $(C_7 O_2)$ which consists of 16,302 molecules. As the chemical property, we choose the band gap energy. 

\subsubsection{Architecture and training}

For the encoder and decoder we use two dense layers including 1024 units with an ELU activation. The property decoder is a fully input-convex neural network. It consists of four hidden layers with 512 units each and a softplus activation. Each of the networks has an additional output layer with a linear activation. The $\beta$ parameter is annealed by a factor of $0.99$ after every 30 epochs and the $\gamma$ constant is set to $0.01$.

\end{document}